\documentclass{article} 
\usepackage{fullpage}
\usepackage[sort&compress,numbers]{natbib}
\usepackage{authblk}
\usepackage[pointedenum]{paralist}
\usepackage{epic,eepic,eepicemu}

\usepackage[utf8]{inputenc} 
\usepackage[T1]{fontenc}    
\usepackage{hyperref}       
\usepackage{url}            
\usepackage{booktabs}       
\usepackage{amsfonts}       
\usepackage{nicefrac}       
\usepackage{microtype}      

\usepackage{times}
\usepackage{graphicx,subfig}
\usepackage{graphics}
\usepackage{tabularx}
\usepackage{multirow}
\usepackage{amsfonts,amsmath,amssymb,amsthm,url,xspace}
\usepackage{natbib}

\usepackage{paralist}

\usepackage[ruled,vlined, linesnumbered]{algorithm2e}
\usepackage{enumitem}
\newcommand{\argmax}{\operatornamewithlimits{argmax}}
\newcommand{\argmin}{\operatornamewithlimits{argmin}}

\newcommand{\bd}{\boldsymbol{d}}
\newcommand{\bb}{\boldsymbol{b}}

\newcommand{\R}{\mathbb{R}}

\newcommand{\bzero}{{\boldsymbol{0}}}

\newcommand{\by}{{\boldsymbol{y}}}
\newcommand{\be}{{\boldsymbol{e}}}
\newcommand{\ba}{{\boldsymbol{a}}}
\newcommand{\bw}{{\boldsymbol{w}}}
\newcommand{\bx}{{\boldsymbol{x}}}

\newcommand{\AL}{\boldsymbol{\alpha}}
\newcommand{\BL}{\boldsymbol{\beta}}

\newcommand{\bp}{{\boldsymbol{p}}}

\def\covtype{{\sf covtype}\xspace}
\def\cifar{{\sf cifar}\xspace}
\def\liblinear{{\sf LIBLINEAR}\xspace }

\def\MNIST{{\sf mnist8m}\xspace}
\def\KDDA{{\sf kdda}\xspace}

\def\covtype{{\sf covtype}\xspace}
\def\webspam{{\sf webspam}\xspace}

\def\DBCD{{PBM}\xspace}

\newtheorem{definition}{Definition}

\newtheorem{corollary}{Corollary}
\newtheorem{theorem}{Theorem}

\title{Communication-Efficient Parallel Block Minimization for Kernel Machines}
\author[1]{Cho-Jui Hsieh}
\author[2]{Si Si} 
\author[2]{Inderjit S. Dhillon}
\affil[1]{Departments of Computer Science and Statistics, University of California, Davis}
\affil[2]{Department of Computer Science, University of Texas, Austin}

\date{}

\begin{document}

\maketitle

\begin{abstract} 

  Kernel machines often yield superior predictive performance on various tasks; however, they suffer from severe computational challenges. In this paper, we show how to overcome the important challenge of speeding up kernel machines. In particular, we develop a parallel block minimization framework for
solving kernel machines, including kernel SVM and kernel logistic regression. 
Our framework proceeds by dividing the problem into smaller subproblems
  by forming a block-diagonal approximation of the Hessian matrix. 
  The subproblems are then solved approximately in parallel. 
  After that, a communication efficient line search procedure is developed to
  ensure sufficient reduction of the objective function value at each iteration. 
  We prove global linear convergence rate of the proposed method 
  with a wide class of subproblem solvers, and our analysis covers
  strongly convex and some non-strongly convex functions. 
  We apply our algorithm to solve large-scale kernel SVM problems on distributed 
  systems, and show a significant improvement over existing parallel solvers. 
  As an example, on the covtype dataset with half-a-million samples, our algorithm 
  can obtain an approximate solution with 96\% accuracy
  in 20 seconds using 32 machines, while 
all the other parallel kernel SVM solvers require more than 2000 seconds to achieve a solution with 95\% accuracy. 
Moreover, our algorithm can scale to very large data sets, such as the kdd algebra dataset with 8 million samples and 20 million features. 
\end{abstract} 

\section{Introduction}
\label{sec:intro}
Kernel methods are a class of algorithms that map samples from input space to a high-dimensional
feature space. 
The representative kernel machines include  kernel SVM, kernel 
logistic regression and support vector regression.
All the above algorithms are hard to scale to large-scale problems
due to the high computation cost and memory requirement 
of computing and storing the kernel matrix. 
Efficient sequential algorithms have been proposed for kernel machines, 
where the most widely-used one is the SMO algorithm~\cite{JP98a,TJ98a} implemented
in LIBSVM and SVMLight. 
However, 
single-machine kernel SVM algorithms still cannot scale to larger datasets (e.g. LIBSVM takes 3 days on the MNIST dataset with 8 million samples), 
%
%
so there is a tremendous need for developing
distributed algorithms for kernel machines. 

There have been some previous attempts in developing distributed kernel SVM solvers~\cite{EC07a,ZAZ09a}, 
but they converge much slower than SMO, 
and cannot achieve ideal scalability using multiple machines.  
On the other hand, 
other recent distributed algorithms 
for solving linear SVM/logistic regression~\cite{MJ14a,YZ15a,CPL15a} 
cannot be directly applied to kernel machines since they
all synchronize information using primal variables. 

In this paper, we propose a Parallel Block Minimization (\DBCD) 
framework for solving kernel machines on distributed systems.
  At each iteration, \DBCD divides the whole problem into smaller subproblems
  by first forming a block-diagonal approximation of the kernel matrix.
  Each subproblem is much smaller than the original problem and they can be solved in parallel using existing serial solvers such as SMO (or greedy coordinate descent)~\cite{JP98a} implemented in LIBSVM. 
  After that, we develop a communication-efficient line search procedure
  to ensure a reduction of the objective function value.

  Our contribution can be summarized below:
\begin{enumerate}[noitemsep,nolistsep,leftmargin=*]
\item We are the first to apply the parallel block minimization (\DBCD) framework for
  training kernel machines, where each machine updates a block of dual 
  variables.  
  We develop an efficient line search procedure to synchronize the updates at the end of each iteration
  by communicating the current prediction on each training sample. 
\item Although our algorithm works for any partition of dual variables, we show 
  that the convergence speed can be improved using a better partition. 
  As an example, we show the algorithm converges faster if we partition
  dual variables using kmeans on a subset of training samples. 
\item We show that our proposed algorithm significantly outperforms existing 
  kernel SVM and logistic regression
  solvers on a distributed system.  
\item We prove a global linear convergence rate of \DBCD 
  under mild conditions: we allow a wide range of inexact subproblem solvers, and our 
  analysis covers many widely-used functions where  some of them may not be strongly convex
  (e.g., SVM dual problem with a positive semi-definite kernel). 
\end{enumerate}
The rest of the paper is outlined as follows. 
We discuss the related work in 
Section~\ref{sec:related}. 
In Section~\ref{sec:setup}, we introduce the problem setting, 
and the proposed \DBCD framework is presented in Section~\ref{sec:proposed}. 
We prove a global linear convergence rate in Section~\ref{sec:convergence}, 
and experimental results are shown in Section~\ref{sec:exp}. We conclude the whole paper in Section~\ref{sec:conclude}.

\section{Related Work}
\label{sec:related}
{\bf Algorithms for solving kernel machines. }
Several optimization algorithms have been proposed to solve kernel machines~\cite{JP98a,TJ98a}. 
Among them, decomposition methods~\cite{TJ98a,SSK03a} have become widely-used in software packages
such as LIBSVM and SVMLight. 
Parallelizing kernel SVM has also been studied in the literature. 
Cascade-SVM~\cite{cascadeSVM} proposed to randomly divide the problem into subproblems, but they still need to solve a global kernel SVM
problem in the final stage. Several parallel algorithms have been
proposed to solve the global kernel SVM problem, such as PSVM~\cite{EC07a} (incomplete
Cholesky factorization + interior point method) and 
P-pack SVM~\cite{ZAZ09a} (distributed SGD). 

Two clustering based approaches have been recently developed: Divide-and-conquer SVM (DC-SVM)~\cite{HSD14} 
and Communication Avoiding SVM (CA-SVM)~\cite{YY15a}. 
In DC-SVM, the global problem is hierarchically divided into smaller subproblems by kmeans or kernel kmeans, 
and the subproblems can be solved independently. The solutions can be used directly
for prediction (called DC-SVM early prediction) or can be used to initialize the global SVM solver. 
Although the subproblems can be solved in parallel, they still need a parallel solver
for the global SVM problem in the final stage. Our idea of clustering is similar to DC-SVM~\cite{HSD14}, but here 
we use it to develop a distributed SVM solver, while DC-SVM is just a single-machine serial algorithm. 
The experimental comparisons between our proposed method and DC-SVM are in Figure~\ref{fig:compare_dcsvm}.
CA-SVM modified the clustering step of DC-SVM early prediction to have more balanced partitions. 
DC-SVM (early prediction) and CA-SVM are variations of DC-SVM that uses local solutions directly for prediction.  
Therefore they cannot get the global SVM optimal solution, moreover, their solutions will become worse when using more computers. 
If we initialize our proposed \DBCD algorithm by $\AL=\bzero$ (where $\AL$ is the dual variables), then the first step of \DBCD is actually
equivalent to the distributed version of DC-SVM early prediction and CA-SVM. However, 
using our algorithm we can further obtain an accurate global SVM solution on a distributed system. 

{\bf Kernel approximation approaches. }
Another line of approach aims to approximate the kernel matrix
and then solve the reduced size problem. Nystr\"{o}m approximation~\cite{CW01a} 
and random Fourier features~\cite{AR08b} can be used to form
low-rank approximation of kernel matrices, and then the problem can be 
reduced to a linear SVM or logistic regression problem. 
To parallelize these algorithms, a distributed linear SVM or logistic regression
solver is needed. The comparisons are in Figure~\ref{fig:comparison}. Using random Fourier features in kernel machine is also applied in~\cite{Huang_Kernel_DNN_ICASSP2014} and ~\cite{TuRVR16}. Comparing with our proposed method, (1) they only consider solving linear systems, and it is nontrivial for them to solve kernel SVM and logistic regression problems; (2) they cannot obtain the exact solution of kernel machines, since they only use the approximate kernel. 

{\bf Distributed linear SVM and logistic regression solvers. }
Several distributed algorithms have been proposed for solving the primal linear SVM problems~\cite{YZ15a}, but they cannot be directly
applied to the dual problem. 
Distributed dual coordinate descent~\cite{TY13a,MJ14a,CPL15a} is also investigated for linear SVM. 

{\bf Distributed Block Coordinate Descent. }
The main idea of our algorithm is similar to 
a class of Distributed Block Coordinate Descent (DBCD) parallel
computing methods. It was discussed recently in many literature~\cite{PR12a,CS12a,DM15a}, 
where each of them have different ways to solve subproblems in each machine, and synchronize the results.

{\bf Our Innovations. }
In contrast to the above algorithms, ours is the first paper that develops a block-partition based coordinate descent algorithm to kernel machines; 
previous work either focused on linear SVM~\cite{CPL15a,TY13a,MJ14a} or primal ERM problems~\cite{DM15a}. 
These approaches proposed to maintain and synchronize the vector $\by = X^T \bd$ or $\bw=X\bd$, where $\bd$ is the direction and
$X$ is the data matrix. 
Unfortunately, in kernel methods this strategy does not work since each sample may have infinite dimension after the nonlinear 
mapping. We overcome the challenges by synchronizing the $Q\bd$ vector ($Q\in \R^{n\times n}$) and 
developing an efficient line search procedure. We further show that a better 
partition (obtained by kmeans) can speedup the convergence of the algorithm. 

On the theoretical front: previous papers~\cite{TY13a,MJ14a} can only show linear convergence when the objective function is
$f(x)=g(x)+h(x)$ and $g$ is strongly convex. \cite{CPL15a} considered some non-strongly convex functions, but they assume each subproblem is solved exactly. In this paper, we prove a global linear convergence even when $g$ is not strongly convex and each subproblem is solved approximately. Our proof covers general DBCD algorithms for some non-strongly convex functions, for which previous analysis can only show sub-linear convergence. 

\section{Problem Setup}
\label{sec:setup}
We focus on the following composite optimization problem: 
\begin{equation}
  \argmin_{\AL\in \R^n} \! \big\{\AL^T \!Q \AL + \sum_i \!
  g_i(\alpha_i)\big\}\!:=\!f(\AL)  \text{ s.t. } \ba \!\leq \!\AL\! \leq\! \bb, 
  \label{eq:main_pb}
\end{equation}
where $Q\in \R^{n\times n}$ is positive semi-definite
and each $g_i$
is a univariate convex function. Note that we can easily handle the box constraint $\ba \leq \AL\leq \bb$ 
by setting
$g_i(\alpha_i)=\infty$ if $\alpha_i\notin[a_i,b_i]$, so we will omit the constraint
in most parts of the paper. 

An important application of~\eqref{eq:main_pb} in machine learning is that it is
the dual problem of $\ell_2$-regularized empirical risk minimization. 
Given a set of instances $\{\bx_i, y_i\}_{i=1}^n$, 
we consider the following $\ell_2$-regularized empirical risk minimization problem: 
\begin{equation}
  \arg\min_{\bw} \frac{1}{2}\bw^T \bw + C \sum\nolimits_{i=1}^n \ell_i( \bw^T \Phi(\bx_i)), 
  \label{eq:erm}
\end{equation}
where $\ell_i$ is the loss function depending on the label $y_i$, and $\Phi(\cdot)$ is 
the feature mapping. For example, $\ell_i(u) = \max(0, 1-y_i u)$  for SVM with hinge loss. 
The dual problem of~\eqref{eq:erm} can be written as
\begin{equation}
  \arg\min_{\AL} \frac{1}{2}\AL^T Q \AL + \sum\nolimits_{i=1}^n \ell_i^*(-\alpha_i),
  \label{eq:erm_dual}
\end{equation}
where $Q\in \R^{n\times n}$ in this case is the kernel matrix with $Q_{ij}=y_i y_j\Phi(\bx_i)^T \Phi(\bx_j)$. 
Our proposed approach works in the general setting, but we will discuss in more detail its applications
to kernel SVM, where $Q$  is the kernel matrix and $\AL$ is the vector
of dual variables. Note that, similar to~\cite{HSD14}, we ignore the bias term in~\eqref{eq:erm}. Indeed, 
in our experimental results we did not observe improvement
in test accuracy by adding the bias term. 

\section{Proposed Algorithm}
\label{sec:proposed}
We describe our proposed framework \DBCD for solving~\eqref{eq:main_pb} on a 
distributed system with $k$ worker machines. We partition the variables $\AL$ into 
$k$ disjoint index sets $\{S_r\}_{r=1}^k$ such that 
\begin{equation*}
  S_1\cup S_2 \cup \cdots \cup S_k = \{1, \dots, n\} \ \ \text{ and } \  \
  S_p\cap S_q = \phi \ \forall p\neq q, 
\end{equation*}
and we use $\pi(i)$ to denote the cluster indicator that $i$ belongs to. 
We associate each worker $r$ 
with a subset of variables $\AL_{S_r} :=\{\alpha_i \mid i\in S_r\}$. 
Note that our framework allows any partition, and we will discuss how to obtain
a better partition in Section~\ref{sec:clustering}. 

At each iteration, we form the quadratic approximation of problem~\eqref{eq:main_pb} around the current solution:
\begin{align}
  f(\AL+\Delta \AL) \! \approx \! \bar{f}_{\AL}(\Delta\AL) \! = \!& \frac{1}{2} \AL^T Q \AL \! + \! \AL^T Q \Delta\AL \! + \!
  \frac{1}{2}\Delta\AL^T \bar{Q} \Delta\AL
    + \sum_i g_i(\alpha_i + \Delta 
  \alpha_i),  
  \label{eq:q_approx}
\end{align}
where the second order term of the quadratic part ($\Delta\AL^T Q \Delta\AL$) is replaced by 
$\Delta\AL^T \bar{Q}\Delta\AL$, 
and
$\bar{Q}$ is the block-diagonal approximation of $Q$ such that 
\begin{equation}
  \bar{Q}_{ij} = \begin{cases} Q_{ij} &\text{ if } \pi(i) = \pi(j) \\
    0 &\text{ otherwise}. 
  \end{cases}
\end{equation}
By solving~\eqref{eq:q_approx}, we obtain the descent direction $\bd$: 
\begin{equation}
  {\bd} :=  \arg\min_{\Delta\AL} \bar{f}_{\AL}(\Delta\AL). 
  \label{eq:d_compute}
\end{equation}
Since $\bar{Q}$ is block-diagonal, problem~\eqref{eq:d_compute} can be decomposed into
$k$ independent subproblems based on the partition $\{S_r\}_{r=1}^k$: 
\begin{equation}
  \hspace{-2pt} \bd_{S_r} \!\! = \! \argmin_{\Delta\AL_{S_r}} \!  \big\{ \!
  \frac{1}{2}\Delta\AL_{S_r}^T Q_{S_r, S_r}\! \Delta\AL_{S_r} \! + \! 
  \sum_{i\in S_r} \! \bar{g}_i(\Delta\alpha_i) \! \big\} \! := \! f^{(r)}_{\AL}(\!\Delta\AL_{S_r}\!) 
  \label{eq:subpb}
\end{equation}
where 
$\bar{g}_i(\Delta\alpha_i) = g_i(\alpha_i + \Delta\alpha_i) + (Q\AL)_{i} \Delta\alpha_i$
. Subproblem~\eqref{eq:subpb} has the same form as the original 
problem~\eqref{eq:main_pb}, so can be solved by any existing solver. 
The descent direction $\bd$ is the concatenation of 
$\bd_{S_1}, \dots \bd_{S_r}$. Since $f(\AL+\bd)$ might even increase the 
objective function value $f(\AL)$, we 
find the step size
$\beta$ to ensure the following
sufficient decrease condition of the objective function value: 
\begin{equation}
  f(\AL + \beta \bd) - f(\AL)\leq  \beta \sigma \Delta, 
  \label{eq:line_search_condition}
\end{equation}
where $\Delta=\nabla f(\AL)^T\bd$, and $\sigma\in(0,1)$ is a constant. We
then update $\AL\leftarrow \AL+\beta \bd$. 
Now we discuss details 
of each step of our algorithm. 
%

\subsection{Solving the Subproblems}
\label{sec:subproblem_solver}
%
Note that subproblem~\eqref{eq:subpb} has the same form as the original 
problem~\eqref{eq:main_pb}, so we 
can apply any existing algorithm to solve it for each worker independently.  
In our implementation, we apply the following greedy coordinate descent method (a similar algorithm was used in LIBSVM). Assume the current subproblem solution is $\Delta\AL_{S_r}$, we choose variable with the largest projected
gradient: 
\begin{align}
  i^* \!:=
  \argmax_{i\in S_r} \big| \Pi \big(\alpha_i\!+\!\Delta\alpha_i \!-\! (Q_{S_r, S_r} \Delta\AL_{S_r})_{i} \!-\! \bar{g}_i'(\Delta\alpha_i)\big) 
    -\! \alpha_i\!-\!\Delta\alpha_i\big|
  \label{eq:selection}
\end{align}
where $\Pi$ is the projection to the interval $[a_i, b_i]$. 
The selection only requires $O(|S_r|)$ 
time if $Q_{S_r, S_r}\Delta\AL_{S_r}$ is maintained in local memory. 
Variable $\Delta\alpha_i$ is then updated by solving the following one-variable 
subproblem:
\begin{align}
  \delta^*  = \argmin_{\delta: a_{i^*} \leq \alpha_{i^*}+\Delta\alpha_{i^*}+\delta \leq b_{i^*}} \ \frac{1}{2}Q_{i^*, i^*}\delta^2 + (Q_{S_r, S_r}\Delta\AL_{S_r})_{i^*} \delta
  + \bar{g}_{i^*} (\Delta\alpha_{i^*}+\delta). 
  \label{eq:one-variable-update}
\end{align}
For kernel SVM, the one-variable subproblem~\eqref{eq:one-variable-update}
has a closed form solution, while for logistic regression
the subproblem can be solved by Newton's method (see~\cite{HFY10a}). 
The bottleneck of both~\eqref{eq:selection} and~\eqref{eq:one-variable-update}
is to compute $Q_{S_r, S_r}\Delta\AL_{S_r}$, which can be maintained
after each update using $O(|S_r|)$ time. 

Note that in our framework each subproblem does not need to be solved exactly. 
In Section \ref{sec:convergence} we will give theoretical analysis to the
in-exact \DBCD method, and show that the linear convergence is guaranteed
when each subproblem runs more than 1 coordinate update. 


\paragraph{Communication Cost. } There is no communication needed for solving 
the subproblems between workers; however, after solving the subproblems and obtaining $\bd$,  
each worker needs to obtain the updated $(Q\bd)_{S_r}$ vector for next iteration. 
Since each worker only has local $\bd_{S_r}$, we compute $Q_{:,S_r}(\bd_{S_r})$ 
in each worker,  
and use a {\sc Reduce\_Scatter} collective communication to obtain updated $(Q\bd)_{S_r}$
for each worker. 
The communication cost for the collective {\sc Reduce\_Scatter} operation 
for an $n$-dimensional 
vector requires
\begin{equation}
  \log(k) T_{\text{initial}} +  n T_{\text{byte}}
\end{equation}
communication time, where $T_{\text{initial}}$ is the message startup time
and $T_{\text{byte}}$ is the transmission time per byte (see Section 6.3 of~\cite{EC07b}). 
When $n$ is large, the second term usually dominates, so we need $O(n)$ communication
time and this {\em does not grow with number of workers}. 


\subsection{Communication-efficient Line Search} 
\label{sec:line_search}
After obtaining $(Q\bd)_{S_r}$ for each worker, we
propose the following efficient line search approaches.

{\bf Armijo-rule based step size selection. }
For general $g_i(\cdot)$, a commonly used line search approach is to 
try step
sizes $\beta \in \{1, \frac{1}{2},  \dots\}$ until $\beta$
satisfies the sufficient decrease condition~\eqref{eq:line_search_condition}. 
The only cost is to evaluate the objective function value. For each choice 
of $\beta$, $f(\AL+\beta\bd)$ can be computed as 
\begin{align*}
  f(\AL + \beta \bd) = f(\AL) + \sum_r \{\beta \bd_{S_r}^T (Q \AL)_{S_r} + \frac{1}{2} \beta^2 \bd^T_{S_r} (Q \bd)_{S_r} 
   + \sum_{i\in S_r} g_i(\alpha_i+\beta d_i)-g_i(\alpha_i)\}, 
\end{align*}
so if each worker has the vector $(Q\bd)_{S_i}$, we can compute
$f(\AL+\beta\bd)$ using $O(n/k)$ time and $O(1)$ 
communication cost. In order to prove
convergence, 
we add another condition that 
$f(\AL+\beta \bd) \leq f(\AL+\bd/k)$, 
although in practice we do not find any difference without adding this condition. 
\begin{figure*}[tb]
  \centering
  \begin{tabular}{cccc}
  \hspace{-10pt}  \subfloat[\webspam obj]{\includegraphics[width=0.25\textwidth]{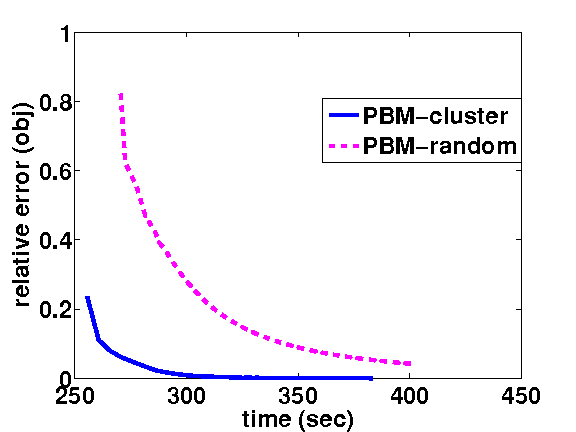}}
  & \hspace{-10pt}
  \subfloat[\webspam accuracy]{\includegraphics[width=0.25\textwidth]{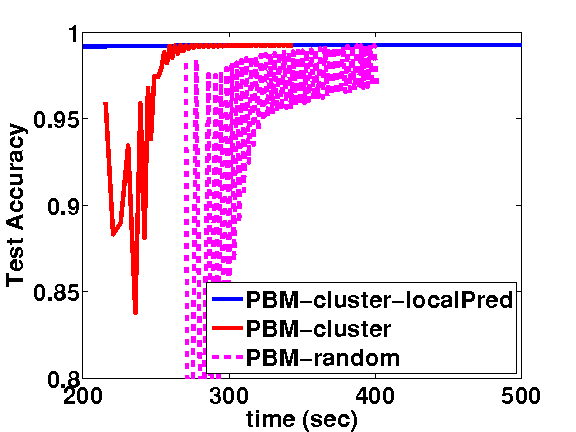}}
  &\hspace{-10pt}
  \subfloat[\covtype obj]{\includegraphics[width=0.25\textwidth]{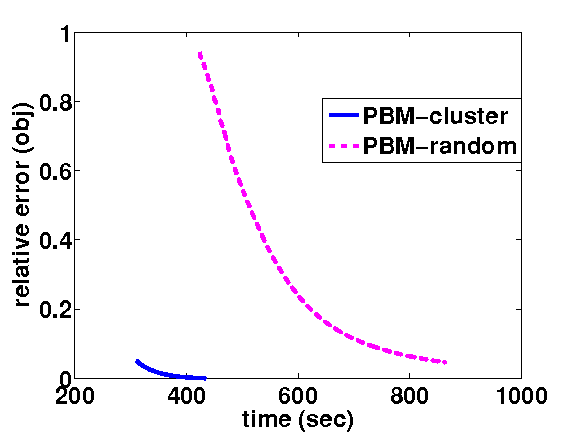}}
  &\hspace{-10pt}
  \subfloat[\covtype accuracy]{\includegraphics[width=0.25\textwidth]{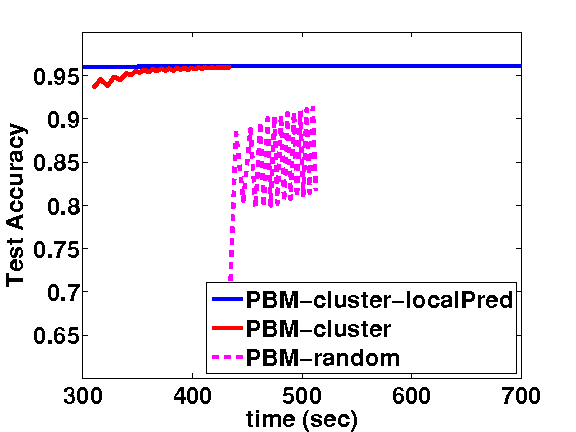}}
\end{tabular}
\caption{\small Comparison of different variances of \DBCD. We observe \DBCD with kmeans partitioning converges faster than random partition, and the accuracy can 
further be improved
by using our local prediction heuristic. 
\label{fig:compare_random} }
\end{figure*}

{\bf Optimal step size selection. }
If each $g_i$ is a linear function with bounded constraint (such as for the 
kernel SVM case), the optimal step size
can be computed without communication. 
The optimal step size is defined by 
\begin{equation}
  \beta^*_t := \arg\min_{\beta} f(\AL + \beta \bd) \text{ s.t. } \ba \leq \AL 
  +\beta \bd\leq \bb. 
\end{equation}
If $\sum_i g_i(\alpha_i) = \bp\AL$, then $f(\AL+\beta\bd)$ with respect to 
$\beta$ is a univariate quadratic function, and thus $\beta^*_t$
can be obtained by the following closed form solution:  
\begin{equation}
  \beta = \min(\bar{\eta}, \max(\underline{\eta}, -(\AL^T Q \bd + \bp^T 
  \bd)/(\bd^T Q \bd) )), 
\end{equation}
where $\bar{\eta} := \min_{i=1}^n (b_i - \alpha_i)$ and 
$\underline{\eta} := \max_{i=1}^n (a_i - \alpha_i )$. This can also be
computed in $O(n/k)$ time and $O(1)$ communication time. 
Our proposed algorithm is summarized in Algorithm~\ref{alg:dbcd}.
    \begin{algorithm}[tbh]
  \DontPrintSemicolon
  \caption{\DBCD: Parallel Block Minimization for solving~\eqref{eq:main_pb}\label{alg:dbcd}}
  \SetKwInOut{Input}{Input}\SetKwInOut{Output}{Output}
  \Input{The objective function~\eqref{eq:main_pb}, initial $\AL_0$.  }
  \Output{The solution $\AL^*$.  }
  Obtain a disjoint index partition $\{S_r\}_{r=1}^k$. \;
  Compute $Q\AL_0$ in parallel ($Q\AL_0=\bzero$ if $\AL_0=0$). \;
  \For{$t=0, 1, \ldots$(until convergence)}{
  Obtain $\bd_{S_r}$ by solving subproblems~\eqref{eq:subpb} {\bf in parallel}. \;
  Compute $Q_{:,S_r}\bd_{S_r}$ {\bf in parallel}. \;
    Use {\sc Reduce\_Scatter} to obtain $(Q\bd)_{S_r}$ in each worker. \;
    Obtain the step size $\beta$ using line search (see Section 
    \ref{sec:line_search} for details) \;
    $\AL_{S_r}\leftarrow \AL_{S_r}+\beta \bd_{S_r}$ and 
    $(Q\AL)_{S_r} \leftarrow (Q\AL)_{S_r}+\beta (Q\bd)_{S_r}$ {\bf in parallel}. \;
  }
\end{algorithm}
\subsection{Variable Partitioning}
\label{sec:clustering}
Our \DBCD algorithm converges for any choice of partition $\{S_r\}_{r=1}^k$. 
However, it is important to select a good partition in order to achieve 
faster convergence. 
Note that if $\bar{Q}=Q$ in subproblem~\eqref{eq:q_approx}, 
then $\bar{f}_{\AL}(\Delta\AL)=f(\AL+\Delta\AL)$, so the algorithm
converges in one iteration. Therefore, to achieve faster convergence, 
we want to minimize the difference between $\bar{f}_{\AL}(\Delta\AL)$ and 
$f(\AL+\Delta\AL)$, and this can be solved by finding a partition $\{S_r\}_{r=1}^k$ 
to minimize error
 $ \|\bar{Q} - Q\|_F^2 = \sum\nolimits_{i,j} Q_{ij}^2 - \sum\nolimits_{r=1}^k\sum\nolimits_{i,j\in S_r} 
  Q_{ij}^2,  
 $
and the minimizer can be obtained by maximizing the second term. However, we also
want to have a balanced partition in order to achieve better parallelization 
speedup. 

The same problem has been encountered in~\cite{SHD14,HSD14} for forming a good kernel approximation. 
They have shown that kernel kmeans can achieve a good partition for general kernel matrix, 
and if the kernels are shift-invariant (e.g., Gaussian, Laplacian, \dots), we can use kmeans 
clustering on the data points to find a balanced partition with small kernel approximation error. 
Since we do not need the ``optimal'' partition, in practice we only run kmeans or kernel kmeans
on a subset of 20000 training samples, so the overhead is very small compared with the whole 
training procedure. 


We observe \DBCD with kmeans partition 
converges much faster compared to random partition.  
In Figure~\ref{fig:compare_random}, we test the \DBCD algorithm
on the kernel SVM problem with Gaussian kernel, and show that the convergence
is much faster when the partition is obtained by kmeans clustering.

%

%
{\bf Local Prediction. }
Inspired by the early prediction strategy discussed in~\cite{HSD14}, 
we propose a local prediction strategy for \DBCD 
when data is partitioned by kmeans clustering. 
Let $\AL_t$ be the solution before the $t$-th iteration of 
Algorithm~\ref{alg:dbcd}, and $\bd_t$ be the solution of the quadratic subproblem
\eqref{eq:q_approx}. The traditional way is to use $\AL_{t+1} = \AL_t + \beta_t\bd_t$
for predicting new data. 

However, we find the following procedure gives better prediction accuracy compared to using $\AL_{t+1}$: 
we first identify the cluster indicator of the test point $\bx$ by choosing the nearest kmeans center. 
If $\bx$ belongs to the $r$-th cluster, we then 
compute the prediction by the local model $\AL_t + (\bd_t)_{<S_r>}$, 
where $(\bd_t)_{<S_r>}$ is an $n$ dimensional vector that sets all the elements outside $S_r$ to be 0. 
Experimental results in Figure~\ref{fig:compare_random} show that this local prediction strategy 
is generally better than predicting by $\AL_{t+1}$. The main reason is that during the optimization
procedure each local machine fits the local data by $\AL_t + (\bd_t)_{<S_r>}$, so the prediction
accuracy is better than the global model. 

{\bf Summary: Computational and Memory Cost:  }
In Section~\ref{sec:subproblem_solver}, we showed that the greedy coordinate descent solver only requires $O(|S_r|)$ time 
complexity per inner iteration. Before communication, we need to compute $Q_{:,S_r}\bd_{S_r}$ in each machine, 
which requires $O(tn)$ time complexity, where $t$ is number of inner iterations. The line search requires only $O(n/k)$ time. 
Therefore, the overall time complexity is $O(tn)$ for totally $O(tk)$ coordinate updates from all workers, 
so the average time per update is $O(n/k)$, which is $k$ times faster than the original greedy coordinate descent (or SMO) algorithm
for kernel SVM, where $k$ is number of machines. 

The time for running kmeans is $O(\bar{n}\bar{d})$ using $k$ computers, where $\bar{n}$ is number subsamples (we set it to 20,000). 
This is a one time cost and is very small comparing to the cost of solving kernel machines. For example, on Covtype dataset
the kmeans step only took 13 seconds, while the overall training time is 772 seconds. 
Note that {\bf we include the clustering time in all the experimental results}. 

{\bf Memory Cost: }
In kernel SVM, the main memory bottleneck
is the memory needed for storing the kernel matrix. 
Without any trick (such as shrinking) to reduce the kernel storage, the space complexity is $O(n^2)$ for kernel SVM (otherwise
we have to recompute kernel values when needed). 
Using our approach, (1) the subproblem solver only requires $O(n^2/k^2)$ space for the sub-matrix of kernel $Q_{S_r, S_r}$
(2) Before synchronization, 
computing $Q_{:, S_r}\bd_{S_r}$ requires $O(n^2/k)$ kernel entries. Therefore, the memory requirement will be reduced from
$n^2$ to $n^2/k$ using our algorithm. However, for large datasets the memory is still not enough. Therefore, similar to the LIBSVM software, 
we maintain a kernel cache to store columns of $Q$, and maintain the cache using the Least Recent Used (LRU)
policy. In short, if memory is not enough, we use the kernel caching technique (implemented in LIBSVM) that computes the kernel values on-the-fly when it is not in the kernel cache and maintains recently used values in memory.

If each machine contains all the training samples, then $Q_{:, S_r}\bd_{S_r}$ 
can be computed in parallel without any communication, and we only need one
REDUCE\_SCATTER mpi operation to gather the results. However, if each machine only contains a subset of samples, 
they have to broadcast $\{\bx_i \mid i\in S_r, d_i\neq 0\}$ to other machines so that each machine $q$ can compute $(Q\bd)_{S_q}$. 
In this case, we do not need to store the whole training data in each machine, 
and the communication time will be proportional to  
number of support vectors, which is usually much smaller than $n$.



\section{Convergence Analysis}
 \label{sec:convergence}
In this section we show that \DBCD has a global linear convergence rate 
under certain mild conditions. 
Note that our result is stronger than some recent theoretical analysis of
distributed coordinate descent. Compared to~\cite{MJ14a}, we show linear convergence
even when the objective function~\eqref{eq:main_pb} is {\it not} strictly 
positive definite (e.g., for SVM with hinge loss), while ~\cite{MJ14a} only has sub-linear
convergence rate for those cases. 
In comparing to~\cite{CPL15a}, they 
assume that the subproblems in each worker are solved exactly, while we allow an approximate
subproblem solver (e.g., coordinate descent with $\geq 1$ steps).

First we assume the objective function satisfies the following property: 
\begin{definition}
  Problem~\eqref{eq:main_pb} admits a ``global error bound'' if there is 
  a constant $\kappa$ such that 
  \begin{equation}
    \|\AL-P_S(\AL)\| \leq \kappa \|T(\AL)-\AL\|, 
    \label{eq:global_error}
  \end{equation}
  where $P_S(\cdot)$ is the Euclidean projection to the set $S$ of
  optimal solutions, and $T:\R^n\rightarrow \R^n$ is the operator defined by 
  \begin{equation}
    T_i(\AL) = \arg\min_{u} f(\AL + (u-\alpha_i)\be_i), \ \ \forall i=1, 
    \dots, n. 
    \label{eq:T_def}
  \end{equation}
  where $\be_i$ is the standard $i$-th unit vector. We say that the algorithm
  satisfies a global error bound from the beginning if~\eqref{eq:global_error}
  holds for the level set $\{\AL\mid f(\AL)\leq f(\bzero)\}$. 
  \label{def:global_error}
\end{definition}
The definition is similar to~\cite{HYD15,PWW13a}.   It has been shown in ~\cite{HYD15,PWW13a} that the global error bound assumption covers many widely used machine learning functions. Any strongly convex function satisfies this definition, and in those cases, $\kappa$ is the condition number. Next, we discuss some problems that admit a global error bound: 
\begin{corollary}
  The algorithm for solving~\eqref{eq:main_pb} satisfies a global error bound 
  from the beginning if one of the following conditions is true. 
  \begin{compactitem}
  \item $Q$ is positive definite (kernel is positive definite).
  \item For all $i=1, \dots, n$, $g_i(\cdot)$ is strongly convex for all the 
    iterates (e.g., dual $\ell_2$-regularized logistic regression with positive semi-definite kernel).  
  \item $Q$ is positive semi-definite and $g_i(\cdot)$ is linear for all $i$ 
    with a box constraint (e.g., dual hinge loss SVM with positive semi-definite 
    kernel). 
  \end{compactitem}
  \label{cor:functions}
\end{corollary}
Corollary~\ref{cor:functions} implies that many widely used machine learning problems, 
including dual formulation of SVM and $\ell_2$-regularized logistic regression,
admit a global error bound from the beginning even when the kernel has a zero eigenvalue.

We do not require the inner solver to obtain the exact solution of~\eqref{eq:subpb}. 
Instead, we define the following condition for the inexact inner solver. 
\begin{definition}
An inexact solver for solving the subproblem~\eqref{eq:subpb} achieves
  a ``local linear improvement'' if the inexact solution $\bd_{S_r}$ satisfies
  \begin{equation}
   E[f_{\AL}^{(r)}(\bd_{S_r})] -  f_{\AL}^{(r)}(\hat{\bd}_{S_r}) \leq
    \eta \big( f_{\AL}^{(r)}(\bzero) - f_{\AL}^{(r)}(\hat{\bd}_{S_r})\big),
    \label{eq:linear_solver}
  \end{equation}
  for all iterates, where $f_{\AL}^{(r)}$ is the subproblem defined in~\eqref{eq:subpb}, $\hat{\bd}_{S_r}:=\arg\min_{\Delta} f_{\AL}^{(r)}(\Delta)$ 
  is the  optimal solution of the subproblem, and $\eta\in(0,1)$ is a constant. 
  Note that we can drop the expectation when the solver is deterministic. 
  \label{def:linear_reduction}
\end{definition}
In the following we list some widely-used subproblem solvers that satisfy Definition~\ref{def:linear_reduction}.
\begin{corollary}
    The following subproblem solvers satisfy 
    Definition~\ref{def:linear_reduction} if the objective function
    admits a global error bound from the beginning: 
    (1) Greedy Coordinate Descent with at least one step.  
    (2) Stochastic coordinate descent with at least one step.
    (3) Cyclic coordinate descent with at least one epoch. 
    \label{cor:solvers}
\end{corollary}

The condition of~\eqref{eq:linear_solver} will be satisfied if an algorithm
has a global linear convergence rate. 
The global linear convergence rate of cyclic coordinate descent has been proved in Section 3 of~\cite{PWW13a}. 
We show global linear convergence for randomized coordinate descent and greedy coordinate descent
in the Appendix.  

We now show our proposed method \DBCD in Algorithm \ref{alg:dbcd} has a global linear convergence rate in the following theorem. 
\begin{theorem}
  Assume (1) the objective function admits a global error bound from the 
  beginning (Definition~\ref{def:global_error}), (2) the inner solver achieves 
  a local linear improvement (Definition~\ref{def:linear_reduction}), (3) 
  $R_{min}=\min_i Q_{ii}\neq 0$,  and (4) the objective function is $L$-Lipschitz
  continuous for the level set. 
  Then the following global linear convergence rate holds:
  \begin{align*}
    E[f(\AL_{t+1})] - f(\AL^*) 
    \leq \big(1-(1-\eta)R_{min}/(kBL\kappa^2 )\big) \big( E[f(\AL_t)]-f(\AL^*)\big), 
  \end{align*}
  where $\AL^*$ is an optimal solution, and $B = \max_{r=1}^k |S_r|$ is the maximum block size. 
  We can drop the expectation when the solver is deterministic. 
  \label{thm:linear_convergence}
\end{theorem}

\begin{proof}
We first define some notations that we will use in the proof. 
Let $\AL_t$ be the current solution of iteration $t$, 
$\bd_t$ is the approximate solution of~\eqref{eq:subpb} satisfying 
Definition~\ref{def:linear_reduction}. For convenience we will omit
the subscript $t$ here (so $\bd:=\bd_t$).
We use $\bd_{S_r}$  to denote the size $|S_r|$ subvector,  
and $\bd_{<S_r>}$  to denote the $n$ dimensional vector with
\begin{equation*}
  \bd_{<S_r>} = \begin{cases}
    d_i &\text{ if } i \in S_r \\
    0 & \text{ otherwise}
  \end{cases}
\end{equation*}

By the definition of our line search procedure described in Section~\ref{sec:line_search}, 
we have
\begin{align*}
  f(\AL_t + \beta\bd) &\leq f(\AL_t + \frac{1}{k} \sum_{r=1}^k \bd_{<S_r>}) \\
  &= f(\frac{1}{k}\sum_{r=1}^k (\AL_t + \bd_{<S_r>})) \\
  &\leq \frac{1}{k} \sum_{r=1}^k f(\AL_t + \bd_{<S_r>}), 
\end{align*}
where the last inequality is from the definition is from the convexity of 
$f(\cdot)$. 
We define $\hat{\bd}$ to be the optimal solution of~\eqref{eq:q_approx} (so each
$\hat{\bd}_{S_r}$ is the optimal solution of the $r$-th 
subproblem~\eqref{eq:subpb}). 
Then we have
\begin{align}
  &f(\AL_t) - f(\AL_t + \bd_t) \\
  \geq&  f(\AL_t) - \frac{1}{k} \sum_{r=1}^k 
  f(\AL_t + \bd_{<S_r>}) \nonumber \\
  =& \frac{1}{k} \sum_{r=1}^k (f(\AL_t) - f(\AL_t+\bd_{<S_r>})) \nonumber\\
  =& \frac{1}{k} \sum_{r=1}^k \bigg(f(\AL_t) - f(\AL_t + \hat{\bd}_{<S_r>}) \nonumber\\ 
  &+ f(\AL_t + \hat{\bd}_{<S_r>}) - f(\AL_t + \bd_{<S_r>})\bigg) \nonumber\\
  \geq& \frac{1}{k} \sum_{r=1}^k (1-\eta) \bigg( f(\AL_t) - f(\AL_t + 
  \hat{\bd}_{<S_r>}) \bigg), 
  \label{eq:g_d1}
\end{align}
where the last inequality is from the local linear improvement of the 
in-exact subproblem solver (Definition \ref{def:linear_reduction}). 
We then define a vector $\bar{\bd}$ where each element is the optimal
solution of the one variable subproblem: 
\begin{equation*}
  \bar{d}_i = T_i(\AL_t)-(\AL_t)_i \ \  
  \forall i=1, \dots, n, 
\end{equation*}
where $T_i(\AL_t)$ was defined in~\eqref{eq:T_def}. 

Since $\hat{\bd}_{S_r}$ is the optimal solution of each subproblem, 
we have 
\begin{equation*}
  f(\AL_t + \hat{\bd}_{<S_r>}) \leq f(\AL_t + \bar{\bd}_{<S_r>}), \ \forall 
  r=1, \dots, k. 
\end{equation*}
Combining with~\eqref{eq:g_d1}  we get
\begin{align*}
  &f(\AL_t) - f(\AL_t + \bd_t) \\
  \geq &
  \frac{1-\eta}{k} \sum_{r=1}^k \big( f(\AL_t) - f(\AL_t + \bar{\bd}_{<S_r>})  
  \big) \\
   \geq& \frac{1-\eta}{k} \sum_{r=1}^k \bigg( f(\AL_t) - f\big( \frac{1}{|S_r|}
  \sum_{i\in S_r}  (\AL_t + \bar{\bd}_{<i>}) \big) \bigg) \\
   \geq& \frac{1-\eta}{k} \sum_{r=1}^k \bigg( 
  f(\AL_t) - \frac{1}{|S_r|} \sum_{i\in S_r} f(\AL_t + \bar{\bd}_{<i>})
  \bigg) \\
  & \quad\quad\quad\quad\quad\text{ (by the convexity of $f$)}\\
   =& \frac{1-\eta}{k} \sum_{r=1}^k\frac{1}{|S_r|} \sum_{i\in S_r}
  \big( f(\AL_t) - f(\AL_t + \bar{\bd}_{<i>})\big) \\
  \geq& \frac{1-\eta}{kB}  \sum_{i=1}^n \big( f(\AL_t) - 
  f(\AL_t+\bar{\bd}_{<i>}) \big), 
\end{align*}
where the last inequality is by the definition of $B$. 

Now consider the one variable problem $f(\AL_t + u \be_i)$. 
By Taylor expansion and the definition of $R_{min}$ we have 
\begin{align*}
  f(\AL_t) &\geq f(\AL_t + \bar{\bd}_{<i>}) + \partial_i 
  f(\AL_t + \bar{\bd}_{<i>}) \times \bar{d}_i + \frac{1}{2} R_{min}\bar{d}_i^2 \\
  &\geq f(\AL_t + \bar{\bd}_{<i>}) + \frac{1}{2} R_{min} \bar{d}_i^2, 
\end{align*}
where the second inequality comes from the fact that $\bar{d}_i$ is the optimal
solution of the one-variable subproblem. 
As a result, 
\begin{equation*}
  \sum_{i=1}^n \bigg( f(\AL_t) - f(\AL_t + \bar{\bd}_{<i>}) \bigg) 
  \geq \sum_{i=1}^n R_{min} \bar{d}_i^2 = R_{min} \|\bar{\bd}\|^2. 
\end{equation*}
Therefore, 
\begin{align*}
  f(\AL_t) - f(\AL_t + \bd_t) &\geq \frac{R_{min} (1-\eta)}{kB} \|\bar{\bd}\|^2 \\
  &= \frac{R_{min}(1-\eta)}{kB} \|T(\AL_t)-\AL_t\|^2 \\
  &\geq \frac{ R_{min}(1-\eta)}{kB\kappa^2} \|P_S(\AL_t) - \AL_t\|^2  \text{ (by~\eqref{eq:global_error})}\\
  &\geq \frac{ L R_{min}(1-\eta)}{kB\kappa^2} \|f(\AL_t) - f(\AL^*)\|. 
 \end{align*}
Therefore, we have
\begin{align*}
  f(\AL_{t+1}) - f(\AL^*) &= f(\AL_t) - \big( f(\AL_t) - f(\AL_{t+1}) \big) - 
  f(\AL^*) \\
  &\leq \bigg(1- \frac{R_{min}(1-\eta)}{kBL\kappa^2}\bigg) \big( 
  f(\AL_t)-f(\AL^*) \big). 
\end{align*}
\end{proof}
Note that we do not make any assumption on the partition used in \DBCD, 
so our analysis works for a wide class of optimization solvers, 
including distributed linear SVM solver in~\cite{CPL15a} where they only consider that case when the subproblems are solved exactly. 

\begin{figure*}[tb]
  \centering
\begin{tabular}{ccc}
  \subfloat[\webspam, comparison]{\includegraphics[width=0.3\textwidth]{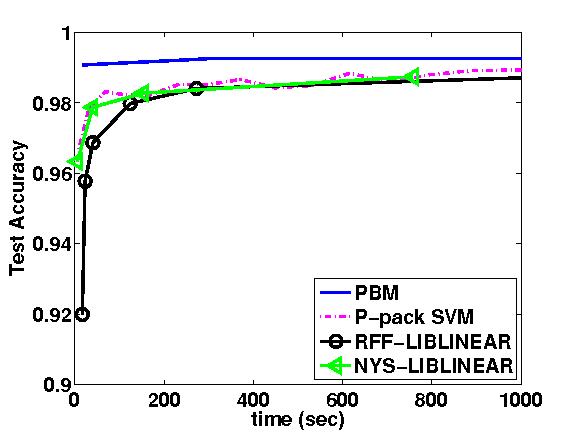}}
  &
  \subfloat[\covtype, comparison]{\includegraphics[width=0.3\textwidth]{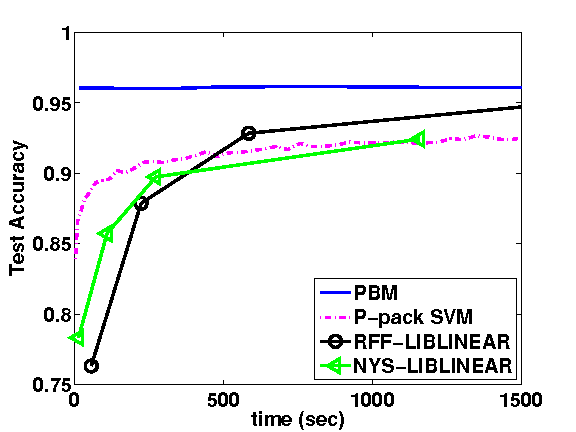}}
  &
  \subfloat[\cifar, comparison]{\includegraphics[width=0.3\textwidth]{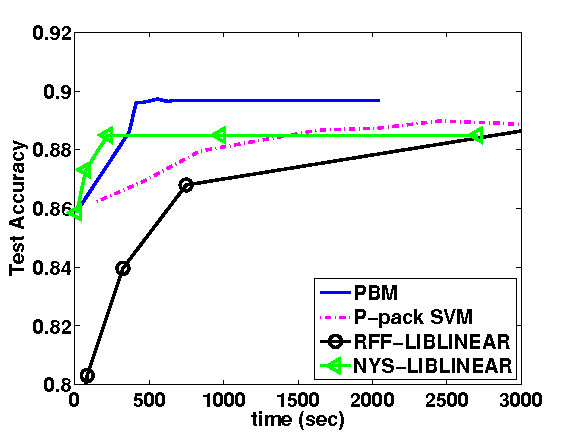}}
  \\ 
  \subfloat[\MNIST, comparison]{\includegraphics[width=0.3\textwidth]{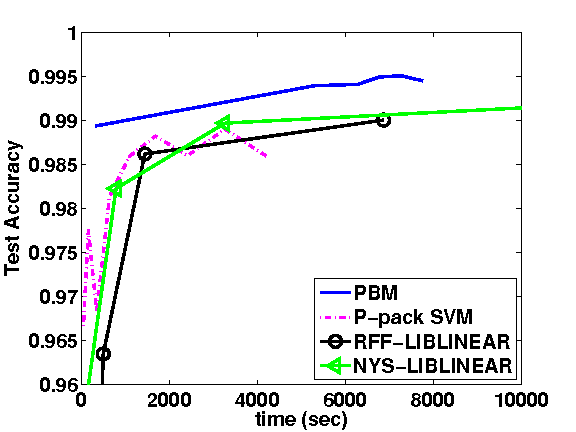}}
  & 
  \subfloat[\webspam, scaling]{\includegraphics[width=0.3\textwidth]{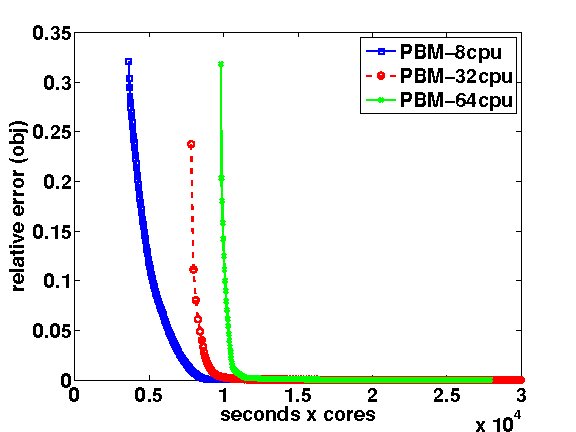}}
  &
  \subfloat[\covtype, scaling]{\includegraphics[width=0.3\textwidth]{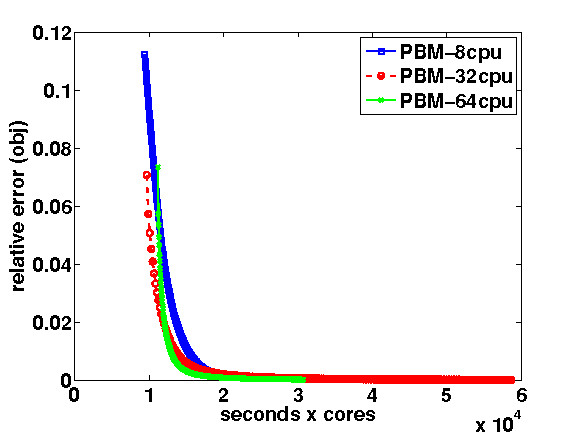}}
\end{tabular}
\caption{\small (a)-(d): Comparison with other distributed kernel SVM solvers using 32 workers. 
Markers for 
RFF-LIBLINEAR and NYS-LIBLINEAR are obtained by 
varying the number random features and landmark points respectively. 
(e)-(f): The objective function of \DBCD as a function of computation time 
(time in seconds $\times$ the number of workers), when the number of workers 
is varied. Results show that \DBCD has good scalability. 
\label{fig:comparison}  }
\end{figure*}

\begin{table}
  \centering
  \caption{Dataset statistics }
  \label{tab:datasets}
\resizebox{12cm}{!}{
  \begin{tabular}{|c|r|r|r|r|r|} \hline
    Dataset & \# training samples & \# testing samples & \# features & $C$ & $\gamma$\\ 
    \hline
\cifar &50,000 & 10,000 & 3072 & $2^3$ & $2^{-22}$\\
    \covtype & 464,810 & 116,202 & 54 & $2^5$ & $2^5$ \\
    \webspam & 280,000 & 70,000 & 254 & $2^3$ & $2^5$\\
\MNIST & 8,000,000 & 100,000 & 784 & $2^0$ & $2^{-21}$ \\
\KDDA & 8,407,752 & 510,302 & 20,216,830 & $2^0$ & $2^{-15}$ \\
    \hline
  \end{tabular}
  }
\end{table}

\begin{table*}
  \centering
  \caption{\small Comparison on real datasets. We use 32 machines for all the distributed solvers (PBM, P-packSGD, PSVM), 
  and 1 machine for the serial solver (DC-SVM). The first column of \DBCD  shows that \DBCD achieves
  good test accuracy after 1 iteration, 
  and the second column of \DBCD shows \DBCD can achieve an 
  accurate solution (with $\frac{f(\AL)-f(\AL^*)}{|f(\AL^*)|}<10^{-3}$) quickly and obtain even better
  accuracy. Note that ``-'' indicates the training time is more than 10 hours and ``x'' indicates the algorithm
  cannot solve logistic regression. The timing for kernel logistic regression (LR) is much slower because $\AL$ is always dense
  when using the logistic loss. Although DC-SVM and PSVM can be potentially used for logistic regression, there 
is no existing implementation for these methods. Writing good code for logistic regression is nontrivial---we have to do kernel caching and shrinking for large datasets. 
Therefore we do not compare with them here.
  \label{tab:comparison}}
  \resizebox{16cm}{!}{
  \begin{tabular}{|c|rr|rr|rr|rr|rr|rr|rr|}
    \hline
    & \multicolumn{2}{|c}{\DBCD (first step)}  & \multicolumn{2}{|c}{\DBCD ($10^{-3}$ error)} & \multicolumn{2}{|c}{P-packSGD}  & \multicolumn{2}{|c}{PSVM $p=n^{0.5}$} & \multicolumn{2}{|c|}{PSVM $p=n^{0.6}$} &  \multicolumn{2}{|c|}{DC-SVM ($10^{-3}$ error)}\\
    \cline{2-13}
  & time(s) & acc(\%)  &  time(s) & acc(\%) & time(s) & acc(\%) & time(s) & acc(\%) & time(s) & acc(\%) & time(s) & acc(\%) \\
    \hline
    \webspam (SVM) & {\bf 16} & 99.07 & 360 & {\bf 99.26} & 1478 & 98.99 &  773 & 75.79 & 2304 & 88.68 & 8742 & 99.26   \\
\covtype (SVM)  & {\bf 14} & 96.05 & 772 & {\bf 96.13} & 1349 & 92.67 & 286 & 76.00 & 7071 & 81.53 & 10457 & 96.13 \\
\cifar (SVM)  & {\bf 15} & 85.91 & 540 & {\bf 89.72} & 1233 & 88.31 & 41 & 79.89 & 1474 & 69.73 & 13006  & 89.72 \\
\MNIST (SVM) & {\bf 321} & 98.94 & 8112 & {\bf 99.45} & 2414 & 98.60 & - & - & - & - & - & - \\
\KDDA (SVM) & {\bf 1832} & 85.51 & 12700 & {\bf 86.31} & - & - & - & - & - & - & - & - \\
    \hline
    \webspam (LR) & {\bf 1679} & 92.01 & 2131 & {\bf 99.07}  & 4417 & 98.96 & x & x & x & x  & x & x\\
    \cifar (LR) & {\bf 471} & 83.37 & 758 & {\bf 88.14} & 2115 & 87.07 & x & x & x & x & x & x \\
    \hline
  \end{tabular}
  }
\end{table*}

\section{Experimental Results}
\label{sec:exp}

We conduct experiments on five large-scale datasets listed in Table~\ref{tab:datasets}. 
We follow the procedure in~\cite{ZK12a,HSD14} to transform \cifar and \MNIST into binary classification problems, 
and Gaussian kernel $K(\bx_i, \bx_j) = e^{-\gamma\|\bx_i-\bx_j\|^2}$ is used in all the comparisons.  
We follow the parameter settings in~\cite{HSD14}, where 
$C$ and $\gamma$ are selected by 5-fold cross validation on a grid  of parameters. 
The experiments are conducted on 
Texas Advanced Computing Center Maverick cluster. 

We compare our \DBCD method with the following distributed kernel SVM training algorithms: 
\begin{enumerate}[noitemsep,nolistsep,leftmargin=*]
\item P-pack SVM~\cite{ZAZ09a}: a parallel Stochastic Gradient Descent (SGD) algorithm 
  for kernel SVM training. We set the pack size $r=100$ according to the original paper. 
\item Random Fourier features with distributed \liblinear: 
In a distributed system, we can compute random features~\cite{AR08b} for each sample
  in parallel (this is a one-time preprocessing step), and then solve the resulting linear SVM problem by distributed
  dual coordinate descent~\cite{CPL15a} implemented in 
  MPI LIBLINEAR. Note that although Fastfood~\cite{QL13a} can generate random features in a faster way, 
  the bottleneck for RFF-LIBLINEAR is solving the resulting linear SVM problem after generating random features, so
  the performance is similar. 

\item Nystr\"{o}m approximation with distributed \liblinear: 
  We implemented the ensemble Nystr\"{o}m approximation~\cite{SK09a} with kmeans sampling in 
  a distributed system and solved the resulting linear SVM problem by MPI 
  LIBLINEAR. The approach is similar to~\cite{DM14a}.
\item PSVM~\cite{EC07a}: a parallel kernel SVM solver by in-complete Cholesky factorization
  and a parallel interior point method. We test the performance of PSVM with 
  the rank suggested by the original paper ($n^{0.5}$ or $n^{0.6}$ where $n$ 
  is number of samples).  
\end{enumerate}
{\bf Comparison with other solvers. } We use 32 machines (each with 1 thread) and the best $C, \gamma$ for all the 
solvers. Our parameter settings are exactly the same with ~\cite{HSD14} chosen by cross-validation in $[2^{-30}, 2^{10}]$. For \cifar, \MNIST and 
\KDDA the samples are not normalized, so the averaged norm mean($\|\bx_i\|$) is large (it is 876 on cifar). Since Gaussian kernel is $e^{-\gamma\|\bx_i-\bx_j\|^2}$, a good $\gamma$ will be very small.  We mainly compare the {\it prediction accuracy } in the paper because most of the parallel kernel SVM solvers are 
``approximate solvers''---they solve an approximated problem, so it is not fair to evaluate them using the original objective function. 
The results in Figure~\ref{fig:comparison} (a)-(d) indicate that
our proposed algorithm is much faster than other approaches. We further test
the algorithms with varied number of workers and parameters in 
Table~\ref{tab:comparison}. 
Note that PSVM usually gets lower test accuracy 
so we only show its results
in Table~\ref{tab:comparison}.  

{\bf Scalability of \DBCD. }For the second experiment we varied the number of workers from 8 to 64, 
and plot the scaling behavior of \DBCD. 
In Figure~\ref{fig:comparison} (e)-(f), we set $y$-axis to be the relative error
defined by $(f(\AL_t)-f(\AL^*))/f(\AL^*)$ where $\AL^*$ is the optimal 
solution, 
and $x$-axis to be the total CPU time expended which is given by the number of 
seconds elapsed multiplied by the number of workers. We plot the convergence
curves by changing number of cores. The perfect linear speedup is achieved
until the curves overlap. This is indeed the case for covtype. 

{\bf Comparison with single-machine solvers. }
We also compare \DBCD with state-of-the-art sequential kernel SVM algorithm 
DC-SVM~\cite{HSD14}. The results are in Table~\ref{tab:comparison} and Figure~\ref{fig:compare_dcsvm}. DC-SVM first computes the solutions in each partition, and then use the concatenation of local (dual) solutions to initialize a global kernel SVM solver (e.g., LIBSVM). However, the top level of DC-SVM is the bottleneck (taking 2/3 of the run time), and LIBSVM cannot run on multiple machines. So DC-SVM cannot be parallelized, and we compare our method with it just to show our distributed algorithm is much faster than the best serial algorithm (which is not the case for many other distributed solvers). 

\begin{figure*}[tb]
  \centering
\begin{tabular}{cccc}
  \hspace{-20pt}  \subfloat[\webspam, obj]{\includegraphics[width=0.25\textwidth]{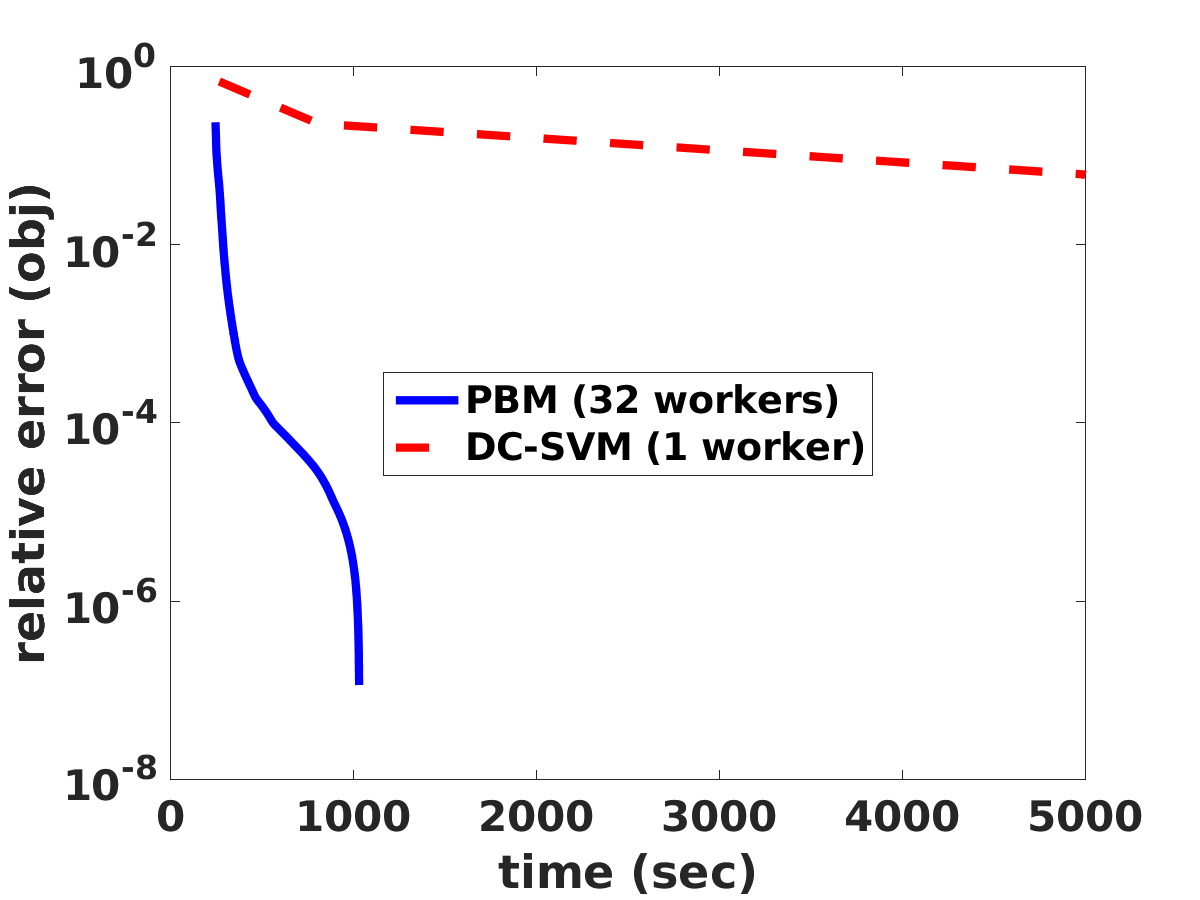}}
  &\hspace{-5pt}
  \subfloat[\webspam, accuracy]{\includegraphics[width=0.25\textwidth]{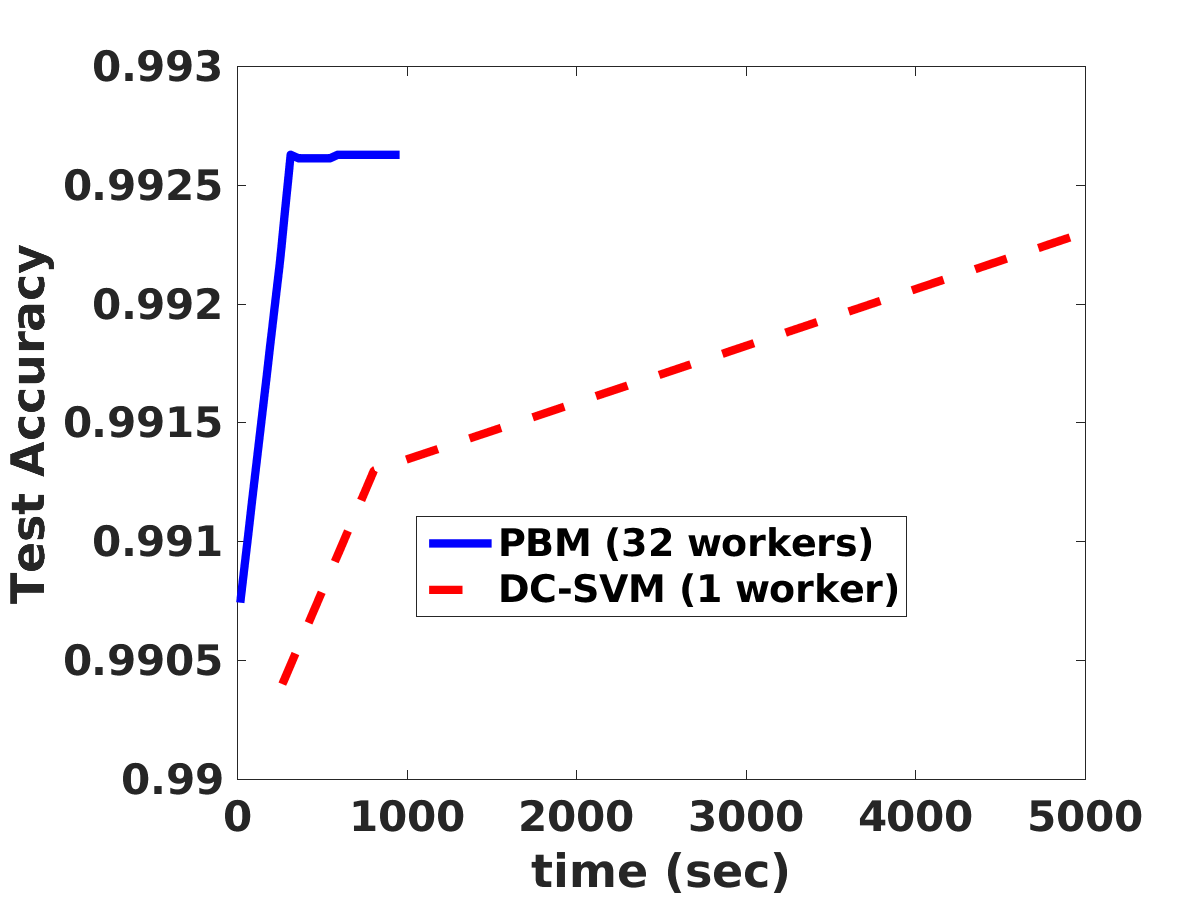}}
  & \hspace{-5pt}
  \subfloat[\covtype, obj]{\includegraphics[width=0.25\textwidth]{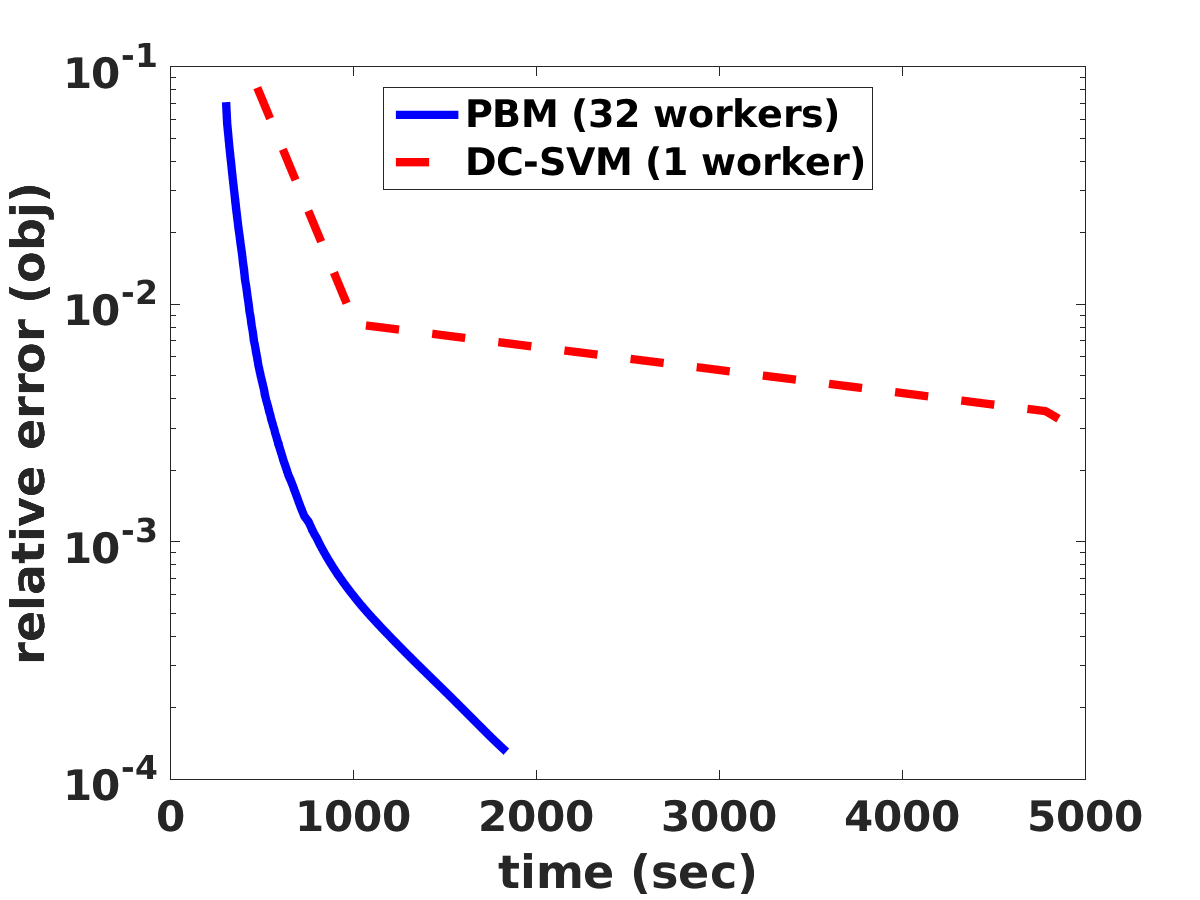}}
  & \hspace{-5pt}
  \subfloat[\covtype, accuracy]{\includegraphics[width=0.25\textwidth]{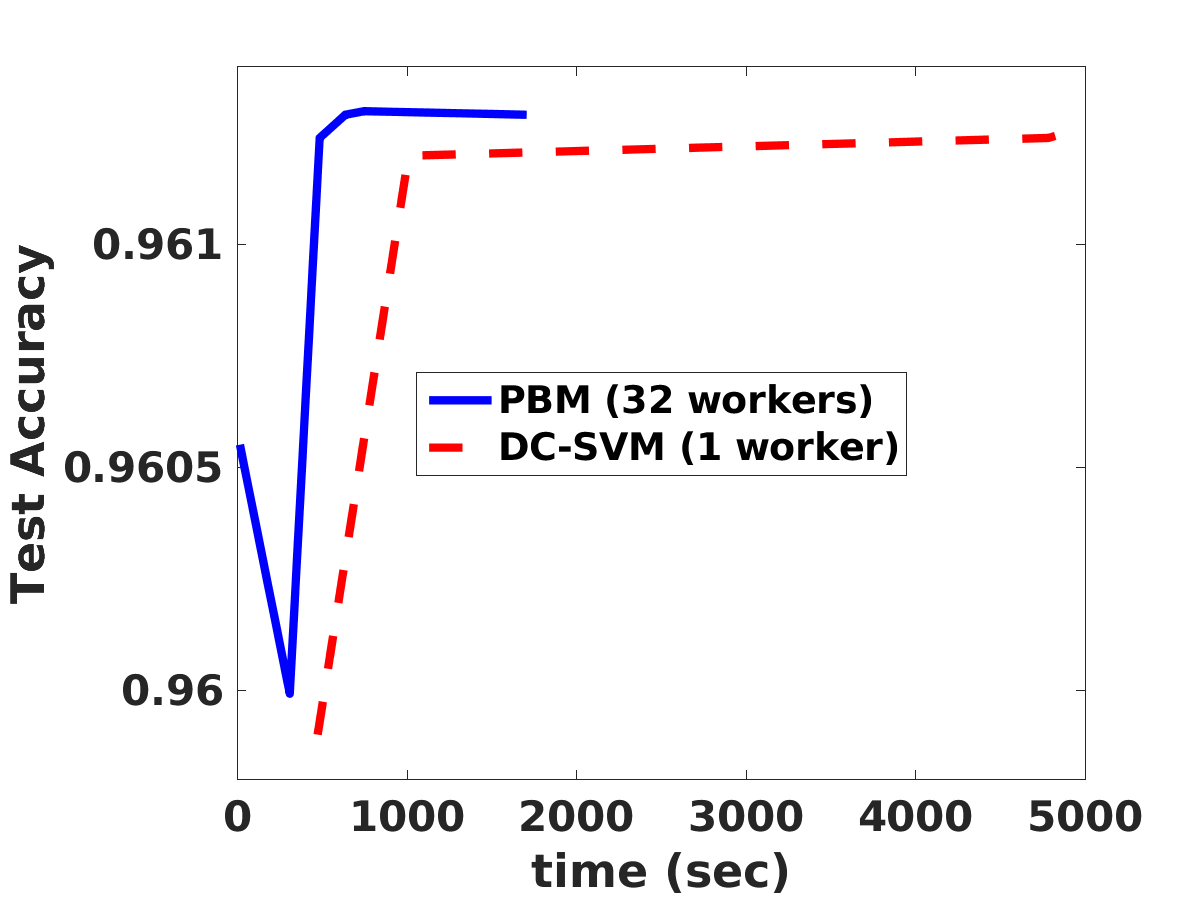}}
\end{tabular}
\caption{ Comparison with DC-SVM (a sequential kernel SVM 
solver). \label{fig:compare_dcsvm}  }
\end{figure*}
 

{\bf Kernel logistic regression. }
We implement the \DBCD algorithm to solve the kernel logistic regression problem. 
Note that PSVM cannot be directly applied to kernel logistic regression. 
We use greedy coordinate descent proposed in~\cite{SSK03a} to solve each subproblem~\eqref{eq:subpb}. 
The results are presented in Table~\ref{tab:comparison}, showing that our algorithm is faster than others.
\section{Conclusion}
\label{sec:conclude}
We have proposed a parallel block minimization (\DBCD) framework for 
solving kernel machines on distributed systems. 
We show that \DBCD significantly outperforms other approaches on large-scale datasets, 
and prove a global linear convergence of \DBCD under mild conditions.
\bibliographystyle{unsrtnat}
\bibliography{sdp}  
\clearpage

\section{Proof of Corollary \ref{cor:solvers}}
The condition of~\eqref{eq:linear_solver} will be satisfied if an algorithm
has a global linear convergence rate. 
The global linear convergence rate of cyclic coordinate descent has been proved in Section 3 of~\cite{PWW13a}. 
We show global linear convergence for Randomized Coordinate Descent (RCD) and 
Greedy Coordinate Descent (GCD)
in the following. In the following we focus on solving the following 
$n$-variate optimization problem by RCD and GCD: 
\begin{equation*}
  \min_{\bx\in\R^{n}} f(\bx), 
\end{equation*}
and we further assume (1) $f(\cdot)$ satisfies the global error bound 
assumption (Definition~\ref{def:global_error}), and 
(2) Each single-variable subproblem $g_i(s) = f(\bx+s\be_i)$ is $R_{min}$-strongly 
convex on $s$ for any $\bx$, and $R_{min}>0$, and (3) $f(\cdot)$ has $L$-Lipchitz 
continuous gradient.  
Clearly the first assumption is satisfied for
the subproblem we considered in this paper with $R_{min}=\min_i q_{ii}$.   

\subsection{Local Linear Improvement for Randomized Coordinate Descent}
We first formally describe the procedure of Randomized Coordinate Descent 
(RCD). 
The RCD algorithm can be 
written below:    

\begin{itemize}
  \item[] For $t=0, 1, \dots$
    \begin{enumerate}
      \item Select an index $i(t)$ uniformly random from $\{1, 2, \dots, n\}$
      \item Compute $\delta_t$ by 
        \begin{equation*}
          \delta_t \leftarrow T_{i(t)} (\AL^t),
        \end{equation*}
        where $T_{i(t)}(\AL^t)$ was defined in~\eqref{eq:T_def}. 
      \item Update $\AL^{t+1} \leftarrow \AL^t + (\delta_t-\alpha^t_{i(t)}) \be_{i(t)}$. 
      \end{enumerate}
  \end{itemize}
At each iteration, RCD selects a coordinate randomly and update it by solving the single
variable subproblem exactly. Now we show a global linear convergence rate of RCD.
Taking $t=1$ in the following theorem we can see RCD satisfies Definition~\ref{def:linear_reduction}. 
\begin{theorem}
  The sequence $\{\AL^t\}$ generated
  by RCD satisfies
  \begin{equation*}
    \bigg( E[f(\AL^{t})] - f(\AL^*) \bigg) \leq \bigg(1- \frac{R_{min}}{2nL\kappa^2} \bigg)^{t}\bigg( E[f(\AL^0)]-f(\AL^*) \bigg), 
  \end{equation*}
  where $\kappa$ is the parameter in the global error bound~\eqref{eq:global_error}, and 
  $L$ is the Lipschitz constant of the gradient. 
\end{theorem}
\begin{proof}
  First, we define another vector $\BL^{t+1}\in \R^{n}$ such that 
  \begin{equation*}
    \beta^{t+1}_j = \begin{cases}
      T_j(\AL^t) &\text{ if } j=i(t) \\
      \alpha^t_j &\text{ if } j\neq i(t). 
    \end{cases}
  \end{equation*}
  We then have
  \begin{align*}
    E_{i(t)} [\|\AL^t - \BL^{t+1}\|^2] &= \sum_{j} \frac{1}{n} (\alpha^t_j - T_j(\AL^t))^2 \\
    & = \frac{1}{n} \|\AL^t - T(\AL^t)\|^2, 
  \end{align*}
  where the expectation is taken with respect to the index $i(t)$. 
  Since the single variable function is $Q_{ii}$-strongly convex and $\alpha^{t+1}_{i}=\beta^{t+1}_i$
  is the optimal solution for the single variable subproblem, we have
  \begin{align*}
    f(\AL^t) - f(\AL^{t+1}) &\geq \frac{Q_{ii}}{2} \|\AL^t - \BL^{t+1}\|^2 \\
    & \geq \frac{R_{min}}{2} \|\AL^t - \BL^{t+1}\|^2. 
  \end{align*}
  Taking expectation on both sides we have
  \begin{align*}
    E[f(\AL^t)] - E[f(\AL^{t+1})] &\geq \frac{R_{min}}{2n} E[\|T(\AL^t)-\AL^t\|^2] \\
    &\geq \frac{R_{min}}{2n\kappa^2} \|\AL^t - P_S(\AL^t)\|^2 \\
          &\geq \frac{R_{min}}{2nL\kappa^2} E[f(\AL^t)-f(\AL^*)] \\
  \end{align*}
  where $\AL^*$ is the optimal solution. 
  Therefore, we have 
  \begin{align*}
    &E[f(\AL^t)] - E[f(\AL^*)] + E[f(\AL^*)] - E[f(\AL^{t+1})] \\
    \geq &\frac{R_{min}}{2nL\kappa^2} \bigg(E[f(\AL^t)]-E[f(\AL^*)]\bigg) 
  \end{align*}
  So 
  \begin{equation*}
    E[f(\AL^{t+1})] - f(\AL^*) \leq 
    (1-\frac{R_{min}}{2nL \kappa^2}) \bigg(E[f(\AL^t)]- E[f(\AL^*)]\bigg). 
  \end{equation*}
\end{proof}

\subsection{Local Linear improvement for Greedy Coordinate Descent}

We formally define the Greedy Coordinate Descent (GCD) updates. 
\begin{itemize}
  \item[] For $t=1, 2, \dots$
    \begin{enumerate}
      \item Select an index $i(t)$ by 
        \begin{equation*}
          i(t) \leftarrow \arg\min_i |T_i(\AL^t)-\alpha^t_i|
        \end{equation*}
      \item Compute $\delta_t$ by 
        \begin{equation*}
          \delta_t \leftarrow T_{i(t)} (\AL^t). 
        \end{equation*}
      \item Update $\AL^{t+1} \leftarrow \AL^t + (\delta_t-\alpha^t_{i(t)}) \be_{i(t)}$. 
      \end{enumerate}
  \end{itemize}
Now we show a global linear convergence rate for greedy coordinate descent. 
Taking $t=1$ in the following theorem we can see GCD satisfies Definition~\ref{def:linear_reduction}. 
\begin{theorem}
  The sequence $\{\AL^t\}$ generated
  by GCD satisfies
  \begin{equation*}
     f(\AL^{t}) - f(\AL^*)  \leq \bigg(1- \frac{ R_{min}}{2nL\kappa^2} \bigg)^{t}\bigg( f(\AL^0)-f(\AL^*) \bigg), 
  \end{equation*}
  where $\kappa$ is the parameter in the global error bound~\eqref{eq:global_error}, and 
  $L$ is the Lipschitz constant of the gradient. 
\end{theorem}
\begin{proof}
  Since $i(t)$ is the maximum absolute value in $T(\AL^t)$, we can derive the following
  inequalities: 
  \begin{align*}
    f(\AL^t) - f(\AL^{t+1}) &\geq \frac{Q_{ii}}{2} (\alpha^{t+1}_{i(t)}-\alpha^t_{i(t)})^2 \\
    &\geq \frac{Q_{ii}}{2}\frac{1}{n} \|T(\AL^t)-\AL^t\|^2 \text{ (by greedy selection rule)}\\
    &\geq \frac{Q_{ii}}{2n\kappa^2}  \|\AL^t - P_S(\AL^{t})\|^2  \text{ (global error bound)}\\
    &\geq \frac{R_{min}}{2nL\kappa^2} \|f(\AL^t) - f(\AL^*)\| 
  \end{align*}
  Therefore, we have 
  \begin{equation*}
    f(\AL^{t+1}) - f(\AL^*) \leq 
    (1-\frac{ R_{min}}{2nL\kappa^2}) \bigg( f(\AL^t)- f(\AL^*)\bigg). 
  \end{equation*}
\end{proof}

%


\end{document}